\documentclass{article}

\usepackage{microtype}
\usepackage{graphicx}
\usepackage{subcaption}
\usepackage{booktabs} % for professional tables

\usepackage{hyperref}

\usepackage[preprint]{icml2026}

\usepackage{amsmath}
\usepackage{amssymb}
\usepackage{mathtools}
\usepackage{amsthm}

\usepackage{booktabs}
\usepackage{array}
\usepackage{makecell}
\usepackage{xcolor}
\usepackage{colortbl}

\definecolor{tokenred}{RGB}{180, 20, 20}
\definecolor{headblue}{RGB}{20, 50, 160}
\definecolor{slablue}{RGB}{235, 240, 255}
\newcommand{\cblue}[1]{\textcolor{headblue}{\textbf{#1}}}
\newcommand{\cred}[1]{\textcolor{tokenred}{\textbf{#1}}}
\definecolor{highlight}{RGB}{0, 102, 204}

\usepackage{tikz}
\usetikzlibrary{positioning, shapes.geometric, arrows.meta, calc, backgrounds, fit, shadows.blur}

\usepackage[capitalize,noabbrev]{cleveref}

\theoremstyle{plain}
\newtheorem{theorem}{Theorem}[section]
\newtheorem{proposition}[theorem]{Proposition}

\theoremstyle{definition}

\theoremstyle{remark}

\usepackage[textsize=tiny]{todonotes}

\icmltitlerunning{Softmax Linear Attention: Reclaiming Global Competition}

\begin{document}

\twocolumn[
  \icmltitle{Softmax Linear Attention: Reclaiming Global Competition}



    \icmlsetsymbol{equal}{*}
    
    \begin{icmlauthorlist}
      \icmlauthor{Mingwei Xu}{equal,yyy}
      \icmlauthor{Xuan Lin}{equal,comp}
      \icmlauthor{Xinnan Guo}{comp}
      \icmlauthor{Wanqing Xu}{comp}
      \icmlauthor{Wanyun Cui}{yyy}
    \end{icmlauthorlist}
    
    \icmlaffiliation{yyy}{MoE Key Laboratory of Interdisciplinary Research of Computation and Economics,Shanghai University of Finance and Economics}
    \icmlaffiliation{comp}{Ant Group}
    
    % 通讯作者（写谁就标谁；可写多个）
    \icmlcorrespondingauthor{Wanyun Cui}{cui.wanyun@sufe.edu.cn}
    
    % 打印脚注：共同一作说明

  % paper; these are used to populate the "keywords" metad\ata in the PDF but
  \icmlkeywords{Machine Learning, ICML}

  \vskip 0.3in
]



\printAffiliationsAndNotice{\icmlEqualContribution}

\begin{abstract}
While linear attention reduces the quadratic complexity of standard Transformers to linear time, it often lags behind in expressivity due to the removal of softmax normalization. This omission eliminates \emph{global competition}, a critical mechanism that enables models to sharply focus on relevant information amidst long-context noise. In this work, we propose \textbf{Softmax Linear Attention (SLA)}, a framework designed to restore this competitive selection without sacrificing efficiency. By lifting the softmax operation from the token level to the head level, SLA leverages attention heads as coarse semantic slots, applying a competitive gating mechanism to dynamically select the most relevant subspaces. This reintroduces the ``winner-take-all'' dynamics essential for precise retrieval and robust long-context understanding. 
Distinct from prior methods that focus on refining local kernel functions, SLA adopts a broader perspective by exploiting the higher-level multi-head aggregation structure. Extensive experiments demonstrate that SLA consistently enhances state-of-the-art linear baselines (RetNet, GLA, GDN) across language modeling and long-context benchmarks, particularly in challenging retrieval scenarios where it significantly boosts robustness against noise, validating its capability to restore precise focus while maintaining linear complexity.
\end{abstract}

% !TeX root = main.tex
\section{Introduction}
\label{sec:intro}

% --- Local macros (kept here to avoid touching the template preamble) ---
% \definecolor{AccentRed}{RGB}{170,0,0}
% \definecolor{AccentBlue}{RGB}{0,70,140}
% \newcommand{\cblue}[1]{\textcolor{AccentBlue}{\textbf{#1}}}
% \newcommand{\cred}[1]{\textcolor{AccentRed}{\textbf{#1}}}
\newcommand{\softmax}{\mathrm{softmax}}
\newcommand{\Attn}{\mathrm{Attn}}
\newcommand{\LA}{\mathrm{LA}}
\newcommand{\softmaxhead}{\softmax_{\mathrm{head}}}

\iffalse
\begin{table}[t]
  \centering
  \small
  \setlength{\tabcolsep}{2pt}
  \renewcommand{\arraystretch}{1.3}
  \caption{A compact comparison of three multi-head attention forms. Full attention relies on an indecomposable \cred{token-wise softmax}; standard linear attention employs a \emph{kernel-based approximation} that sacrifices competition; we propose to leverage the \emph{multi-head architecture} to retain linearity while reintroducing \cblue{inter-head softmax competition} to approximate the selectivity of the original softmax.}
  \label{tab:attn-formula}
  \begin{tabular}{@{}p{0.98\columnwidth}@{}}
    \toprule
    \textbf{Multi-head softmax attention} \\
    {\centering
    \(
      O = \bigoplus_{h=1}^{H} \left(
      \underbrace{\softmax_{t}\left(Q_{h}K_{h}^\top\right)}_{\cred{\text{token-wise (undecomposed)}}}
      V_{h} \right) W^O
    \) \par} \\
    \midrule
    \textbf{Multi-head linear attention} \\
    {\centering
    \(
      O = \bigoplus_{h=1}^{H} \left(
      \underbrace{\phi(Q_{h}) \big(\phi(K_{h})^\top V_{h}\big)}_{\text{kernel-based approximation}}
      \right) W^O
    \) \par} \\
    \midrule
    \textbf{Multi-head \textbf{softmax} linear attention (ours)} \\
    {\centering
    \(
      O = \bigoplus_{h=1}^{H} \left(
      \underbrace{\softmax_{h}(Q'_h) \softmax_{h}(K'_h)}_{\cblue{\text{head-level competition}}}
      \cdot
      \left( \phi(Q_{h}) \phi(K_{h})^\top V_{h} \right)
      \right) W^O
    \) \par} \\
    \bottomrule
  \end{tabular}
\end{table}
\fi

\begin{table*}[t]
  \centering
  \caption{A compact comparison of three multi-head attention forms. Full attention relies on an indecomposable \cred{token-wise softmax}; standard linear attention employs a \emph{kernel-based approximation} that sacrifices competition; we propose to leverage the \emph{multi-head architecture} to retain linearity while reintroducing \cblue{inter-head softmax competition} to approximate the selectivity of the original softmax. $\bigoplus$ denotes vector concatenation.}
  \label{tab:comparison}
  \vspace{2mm}
  \begin{small}
  \begin{tabular}{@{}lllccl@{}}
  \toprule
  \textbf{Mechanism} & \textbf{Attention Formulation} & \textbf{Competition Level} & \textbf{Complexity} & \textbf{State Size} \\ \midrule
  \makecell[l]{Multi-Head \\ Attention} & $\bigoplus_{h=1}^H \underbrace{\text{softmax}_t \left( \frac{Q_h K_h^\top}{\sqrt{d}} \right)}_{\text{over tokens}} V_h$ & Token-wise (Global)  & $O(L^2)$ & $O(L)$ \\ \addlinespace[2mm]
  \makecell[l]{Linear \\ Attention} & $\bigoplus_{h=1}^H \phi(Q_h) \left( \phi(K_h)^\top V_h \right)$ & None (Point-wise) & $O(L)$ & $O(1)$ \\ \addlinespace[3mm]
  \makecell[l]{\textbf{Softmax Linear} \\ \textbf{Attention}} & \(
      \bigoplus_{h=1}^{H} 
      \underbrace{\softmax_{h}(Q'_h) \softmax_{h}(K'_h)}_{\cblue{\text{over heads}}}
      \cdot
      \left( \phi(Q_{h}) \phi(K_{h})^\top V_{h} \right)
    \) & \textbf{Head-wise (Global)} & $O(L)$ & $O(1)$ \\ \bottomrule
  \end{tabular}
  \end{small}
  \end{table*}

Self-attention stands as the backbone of modern Large Language Models (LLMs) but is burdened by quadratic computational complexity due to its global softmax normalization~\citep{vaswani2017attention,lin2022survey}. The efficacy of standard attention stems from its ability to route information through a \emph{globally competitive} distribution, specifically \(\softmax(QK^\top/\sqrt{d})V\), where $d$ is the per-head dimension. However, this softmax term couples every query-key pair, enforcing an \(O(L^2)\) computational cost and precluding the decoupling of history required for constant-memory inference. Consequently, scaling to ultra-long contexts becomes computationally prohibitive.

Linear attention approaches mitigate this bottleneck by decomposing the attention map, though at the cost of removing the softmax normalization. Methods such as those by \citet{katharopoulos2020transformers} and \citet{shen2021efficient} approximate the non-decomposable softmax via kernel feature maps, \(\softmax(QK^\top) \approx \phi(Q)\phi(K)^\top\). This kernel decomposition enables the exploitation of matrix multiplication associativity, reducing complexity to linear \(O(L)\). While efficient, this linearization fundamentally alters the attention mechanism~\citep{mongaras2025expressiveness}.

Eliminating the softmax function is not without consequence; it removes the mechanism of \emph{global competition} that is essential for precise information retrieval. In full attention, the softmax denominator compels all tokens to compete for probability mass, enabling the model to sharply focus on relevant tokens while suppressing noise. In contrast, standard linear attention reduces attention to independent point-wise similarity scores. This absence of competition leads to inherent deficiencies, most notably \emph{Magnitude Neglect}~\citep{fan2025rectifying}. In standard softmax attention, a larger query magnitude sharpens the probability distribution, allowing the model to express high confidence. Conversely, linear attention weights remain invariant to query scaling, rendering the model incapable of dynamically adjusting its focus. Other resulting issues include \emph{Loss of Polarity}~\citep{Meng2025polaformer} and \emph{Context Collapse}~\citep{zhang2026mhla}.

Current efforts to restore expressivity generally fall into two categories. One line of work refines the kernel approximation, proposing magnitude-aware updates~\citep{fan2025rectifying} or polarity-aware kernels~\citep{Meng2025polaformer}. Another mainstream direction introduces data-dependent gating mechanisms~\citep{sun2023retentive,yang2023gated,schlag2021linear} to modulate information flow. However, these approaches remain confined to kernel-based linear approximation. While they enhance feature representation, the linear constraint fundamentally prevents the reintroduction of token-level global competition, leaving the model without a mechanism to strictly normalize and select information.

We address this limitation by shifting focus from single-head decomposition to exploiting the multi-head architecture to introduce global competition at the head level. We argue that precise token-level competition is not always necessary; coarse-grained competition is often sufficient for effective retrieval. Crucially, the multi-head architecture naturally provides the structural basis for this coarser granularity. It is well-established that different attention heads specialize in distinct functional or semantic roles~\citep{voita2019analyzing,Clark2019whatbert,basile2025headpursuit}. By leveraging this inherent specialization, we can impose competition across heads. This allows the model to dynamically prioritize the most relevant semantic subspaces, recovering the necessary selectivity without breaking the linear complexity of the attention mechanism.

Guided by this insight, we introduce \textbf{Softmax Linear Attention (SLA)}, which formulates attention by re-defining the multi-head aggregation mechanism. Instead of forcing a complex kernel decomposition, we exploit the multi-head structure to stratify the mechanism into two levels: (1) \textbf{Intra-head linearity}, where standard kernel decomposition ensures efficiency within each subspace; and (2) \textbf{Inter-head competition}, where a global softmax gate dynamically weights these subspaces. We achieve this by simply superimposing a head-level softmax on top of standard linear attention, effectively replacing the linear concatenation with a competitive ``winner-take-all'' selection. This modification is elegantly minimal and computationally negligible, yet it fundamentally alters the dynamics: it empowers the model to dynamically prioritize the most relevant semantic subspaces, thereby recovering sharp selectivity without sacrificing linear efficiency.

Our contributions are summarized as follows:
\begin{itemize}
  \item We identify that the independent aggregation of multi-head outputs in standard linear attention fundamentally hinders precise information selection. We propose shifting the competition paradigm from \emph{token-wise} to \emph{head-wise}, utilizing the head dimension as a proxy for semantic subspaces to recover global selectivity without quadratic cost.
  \item We introduce SLA, a generalizable mechanism that superimposes inter-head softmax gates on top of standard linear backbones. This stratifies attention into intra-head linearity and inter-head competition, efficiently combining linear complexity with competitive dynamics.
  \item We provide theoretical analysis proving that SLA restores \emph{magnitude sensitivity} and enables asymptotic \emph{winner-take-all} dynamics. This formally resolves the ``Magnitude Neglect'' pathology inherent in standard linear attention, allowing the model to dynamically sharpen its focus based on confidence.
\end{itemize}

\section{Related Work}
\label{sec:related}

\subsection{Efficient Transformers and Linear Attention}
A large body of work seeks to reduce the quadratic complexity of full softmax attention. Kernel-based approaches~\citep{katharopoulos2020transformers,choromanski2021rethinking,Wang2020linformer} approximate the attention matrix via feature maps \(\phi(\cdot)\), enabling \(O(L)\) complexity. Recently, this direction has converged with RNNs and State Space Models (SSMs). Models like RetNet~\citep{sun2023retentive}, RWKV~\citep{peng2023rwkv}, and Mamba~\citep{gu2024mamba} introduce decay or time-varying gates to linear recurrence, significantly improving performance. Further enhancements such as Gated Linear Attention (GLA)~\citep{yang2023gated} and DeltaNet~\citep{yang2024deltanet} incorporate data-dependent gating and delta-rule updates. Despite these innovations, they fundamentally remove the softmax normalization to achieve linearity.

\subsection{The Cost of Removing Softmax: Expressivity Bottlenecks}
Recent studies highlight systematic expressivity gaps when the softmax normalization is removed.
\textbf{Magnitude Neglect.} \citet{fan2025rectifying} show that standard linear attention is insensitive to query norm, preventing sharp focus. They propose magnitude-aware kernels to re-introduce norm dependency.
\textbf{Loss of Polarity.} \citet{Meng2025polaformer} note that non-negative kernels discard negative correlations and propose polarity-aware mechanisms to recover inhibitory signals.
\textbf{Context Collapse.} \citet{Dong2024bridging} and \citet{zhang2026mhla} demonstrate that linear attention maps are often non-injective, leading to context collapse where distinct queries map to identical outputs.
These works focus on repairing specific deficits (magnitude, polarity, diversity) at the kernel level but lack a mechanism to reinstate global competition.

\subsection{Multi-Head Mechanism and Feature Diversity}
Multi-head attention allows models to capture diverse relations in parallel subspaces~\citep{vaswani2017attention}, with heads often specializing in distinct linguistic functions~\citep{Clark2019whatbert,voita2019analyzing}. While standard attention uses softmax within each head, linear attention typically aggregates heads via simple concatenation, lacking inter-head interaction.
Prior works like Talking-Heads Attention~\citep{Shazeer2020talking} explore head mixing, but mostly in the context of \(O(L^2)\) attention. Our work introduces \textbf{inter-head competition} to linear attention. By treating heads as semantic slots and applying a head-wise softmax, we recover global, competition-driven selection while retaining linear complexity.
\section{Method}
\label{sec:method}

%In this section, we first analyze the structural limitations of standard linear attention—specifically, how the removal of the normalization term eliminates competitive selection. We then introduce Softmax Linear Attention (SLA), a simple yet effective mechanism that restores global competition at the head level without sacrificing linear complexity.

\subsection{The Competition Gap in Linear Attention}
\label{sec:method:gap}

Standard multi-head attention (MHA) relies on the softmax function to induce a \emph{competitive} distribution over the context. Formally, for a query \(q\) and a set of keys \(K\) and values \(V\), the output is:
\begin{equation}
    \label{eq:softmax_attn}
    O = \softmax\left(\frac{q K^\top}{\sqrt{d}}\right) V = \sum_{t=1}^{L} \underbrace{\frac{\exp(q \cdot k_t / \sqrt{d})}{\sum_{j=1}^{L} \exp(q \cdot k_j / \sqrt{d})}}_{\alpha_t} v_t.
\end{equation}
The denominator \(\sum_{j} \exp(\cdot)\) is crucial: it enforces a global constraint \(\sum \alpha_t = 1\), compelling all tokens to \emph{compete} for the limited probability mass. This allows the model to sharply focus on a few relevant tokens (high \(\alpha_t\)) while effectively suppressing irrelevant noise (near-zero \(\alpha_t\)), a property we term global competition. Without this competition, the model fails to achieve the ``winner-take-all'' dynamics required for near one-hot distributions, fundamentally limiting its ability to perform precise retrieval tasks where exact token matching is critical.

Linear attention linearizes this process by removing the softmax and utilizing a kernel feature map \(\phi(\cdot)\):
\begin{equation}
    \label{eq:linear_attn}
    O \approx \phi(q) \sum_{t=1}^{L} \phi(k_t)^\top v_t.
\end{equation}
While this enables \(O(L)\) complexity via the associativity of matrix multiplication, it fundamentally alters the information flow. The attention weight for token \(t\) is roughly \(\phi(q)^\top \phi(k_t)\), which is computed independently of other tokens. Absent the global normalization term, there is no mechanism to suppress a token based on the relevance of others. This lack of competition leads to a diffuse attention distribution (``magnitude neglect''), causing linear attention to struggle with precise retrieval tasks where distinguishing the ``best'' match from many ``good'' matches is critical.

\subsection{Softmax Linear Attention}
\label{sec:method:sla}

To restore global competition without reverting to quadratic complexity, we propose Softmax Linear Attention. Our core insight is to shift the competitive mechanism from the fine-grained token level to a coarser semantic cluster level. We leverage the existing multi-head architecture, where each head is known to capture distinct semantic features~\citep{voita2019analyzing,Clark2019whatbert}.

Conventionally, linear attention models treat these heads independently, simply concatenating their outputs via a linear projection. This implies a flat summation of features, discarding the opportunity for competition. However, we observe that introducing competition at this level is computationally affordable: unlike the quadratic token count \(L^2\), the number of heads \(H\) is small and constant. This allows us to reintroduce a softmax-based selection mechanism over heads with negligible cost.

Based on this, we introduce a competitive gating mechanism that operates over these head-slots. Specifically, we decouple the global softmax selection into two symmetric processes: \emph{read competition} (which heads should the query read from?) and \emph{write competition} (which heads should the key write to?). 

Intuitively, this dual gating reflects a symmetric selection process. The \emph{write competition} acts as a router for incoming information: it forces the key to decide which semantic subspaces it belongs to, ensuring that history is stored in the most relevant heads rather than being smeared across the entire memory state. The \emph{read competition} acts as a filter for retrieval: it allows the query to dynamically prioritize specific subspaces based on the current context, ignoring heads that contain irrelevant information. Together, they maintain a ``sharp'' memory access pattern that standard linear attention lacks.

Formally, this leads to a dual gating formulation:
\begin{equation}
    \label{eq:sla_final}
    O_{\text{SLA}} = \bigoplus_{h=1}^{H} \left( (\mathcal{G}^Q_h \odot \phi(Q_h)) (\mathcal{G}^K_h \odot \phi(K_h))^\top V_h \right) W^O.
\end{equation}
Here, \(\mathcal{G}^Q_h \in \mathbb{R}^{L \times 1}\) and \(\mathcal{G}^K_h \in \mathbb{R}^{L \times 1}\) are the head-level softmax gates for query and key, respectively, computed as:
\begin{equation}
    \mathcal{G}^Q_h = \softmax(Q W_{GQ})_h, \quad \mathcal{G}^K_h = \softmax(K W_{GK})_h,
\end{equation}
where \(W_{GQ}, W_{GK} \in \mathbb{R}^{d \times H}\) project the inputs into head-specific importance scores. Crucially, the \(\softmax\) normalization is performed across the head dimension \(H\).

This design is motivated by a low-rank approximation of the full attention matrix. The full softmax term \(\softmax(QK^\top)\) represents a joint probability distribution over tokens. By decoupling it into \(\mathcal{G}^Q(q) \cdot \mathcal{G}^K(k)\), we effectively approximate the joint distribution \(P(h|q, k)\) via independent marginals \(P(h|q) P(h|k)\) in the latent head space. This allows the model to select relevant semantic subspaces based on both the query's intent \(\mathcal{G}^Q_h\) (``what to look for?'') and the key's content \(\mathcal{G}^K_h\) (``what is stored?''), recovering a symmetric, competition-driven selection mechanism similar to the original dot-product attention but at a coarser granularity.

\iffalse
\paragraph{Connection and Difference with Efficient Attention.}
While our method shares the spirit of decomposing \(\softmax(QK^\top)\) with \emph{Efficient Attention} \citep{shen2021efficient}, there is a fundamental structural difference. 
Efficient Attention applies normalization functions \(\rho(\cdot)\) (typically feature-wise softmax) \emph{independently} within each attention head:
\begin{equation}
    \text{EfficientAttn}_h = \rho(Q_h) (\rho(K_h)^\top V_h).
\end{equation}
This independence fails to capture the relative importance variance across different heads, often leading to the loss of ``magnitude awareness.''
In contrast, our approach explicitly incorporates the head dimension \(H\) into the normalization process:
\begin{equation}
    \text{Ours}_h = \underbrace{\sigma_{\text{head}}(Q'_h) \sigma_{\text{head}}(K'_h)}_{\text{Head Modulation}} \cdot (\phi(Q_h) \phi(K_h)^\top V_h).
\end{equation}
By performing \textbf{softmax over heads}, we re-introduce an inter-head competition mechanism, allowing the model to dynamically up-weight informative heads and suppress noisy ones, a critical property of the original softmax attention that is typically absent in linear variants.
\fi

\paragraph{Recurrent Implementation.}
The core advantage of SLA lies in its efficient implementation. Since the gating \(\mathcal{G}^Q, \mathcal{G}^K\) depends only on the current token's input, the mechanism remains compatible with recurrent updates. See Appendix for detailed derivation.
Specifically, for each head \(h\), we maintain a recurrent state \(S_h \in \mathbb{R}^{d \times d}\). At each time step \(t\):
\begin{align}
    \label{eq:sla_update}
    S_{h, t} &= S_{h, t-1} + (\mathcal{G}^K_{h,t} \cdot \phi(k_{h,t}))^\top v_{h,t}, \\
    y_{h, t} &= (\mathcal{G}^Q_{h,t} \cdot \phi(q_{h,t})) S_{h, t}.
\end{align}
Here, \(\mathcal{G}^K_{h,t}\) modulates the strength of the key entering the memory (write competition), while \(\mathcal{G}^Q_{h,t}\) modulates the query strength reading from it (read competition).
The final output at step \(t\) is the concatenation of these per-head readouts:
\begin{equation}
    o_t = \text{Concat}_{h=1}^H \big( y_{h,t} \big) W^O.
\end{equation}

\paragraph{Chunkwise Parallel Training.}
To accelerate training on GPUs, we avoid the sequential bottleneck of the recurrent form by using a chunkwise parallel strategy. The sequence of length \(L\) is split into chunks of size \(C\). Within each chunk, we compute attention via standard matrix multiplication; cross-chunk dependencies are handled by passing the recurrent state \(S\) between chunks. Since the softmax gates \(\mathcal{G}\) are token-local, head-wise scalar modulators that only rescale the per-head read/write strengths, they introduce no additional token-to-token coupling and thus do not disrupt the chunkwise linearization of the \(K,V\) accumulation. As a result, SLA preserves the high training throughput typical of modern linear attention models, while adding competitive selection across heads.

\paragraph{Parameter Efficiency.}
The introduced gating mechanism is extremely lightweight. In terms of parameter count, SLA only adds two projection matrices \(W_{GQ}, W_{GK} \in \mathbb{R}^{d \times H}\) per layer. For the 340M-parameter model used in our experiments (24 layers, \(hidden\_size=1024\)), this results in approximately 0.05M additional parameters, which is negligible (\(<0.02\%\)) relative to the total model size.

\section{Theoretical Analysis}
\label{sec:analysis}

In this section, we provide a theoretical justification for SLA. We analyze its advantages over standard linear attention from two primary perspectives: (1) restoring magnitude sensitivity, and (2) enabling asymptotic winner-take-all selection.

\subsection{Restoring Magnitude Sensitivity via Head Competition}
\label{sec:analysis:magnitude}

A critical deficiency in standard linear attention is \emph{magnitude neglect}: the sharpness of the attention distribution is independent of the query norm. In standard softmax attention, the term \(\softmax(\lambda q \cdot K^\top)\) becomes sharper (approaching a one-hot distribution) as \(\lambda \to \infty\), allowing the model to express high confidence. Conversely, for linear attention formulated as \(\phi(q)\phi(K)^\top\), scaling \(q\) by \(\lambda\) merely scales the output vector by \(\lambda\) without altering the \emph{relative} distribution of attention weights.

\begin{proposition}
    \label{prop:magnitude}
    Let \(w_{\text{lin}}(k) = \frac{\phi(q)^\top \phi(k)}{\sum_j \phi(q)^\top \phi(k_j)}\) be the normalized attention weight for key \(k\) in linear attention. If \(\phi(\cdot)\) is homogeneous (e.g., ReLU), then \(w_{\text{lin}}(k)\) is invariant to scalar scaling of \(q\). Consequently, the entropy of the attention distribution \(\mathcal{H}(w_{\text{lin}})\) remains constant with respect to \(\|q\|\).
\end{proposition}

This invariance prevents the model from dynamically adjusting its focus: it cannot ``concentrate'' attention even when the query is highly confident.

SLA restores this magnitude sensitivity through the head-level softmax gates \(\mathcal{G}^Q_h = \softmax((Q W_{GQ})_h)\) and \(\mathcal{G}^K_h = \softmax((K W_{GK})_h)\).

\begin{theorem}
    \label{thm:magnitude_recovery}
    In SLA, the head gating distribution \(\mathcal{G}^Q\) is sensitive to the query magnitude. Specifically, if we scale the projection \(s = x W_{GQ}\) by \(\lambda\), the entropy of the gating distribution \(\mathcal{H}(\mathcal{G}^Q)\) decreases as \(\lambda\) increases.
    \begin{equation}
      \lim_{\lambda \to \infty} \mathcal{G}^Q(\lambda s) = \text{one\_hot}(\operatorname*{argmax}_h s_h).
    \end{equation}
\end{theorem}
A similar property holds for \(\mathcal{G}^K\) with respect to key magnitude. This implies that SLA recovers the ability to perform confidence-based sharpening. When the model is confident (large query magnitude), the softmax gate saturates, selecting a single head (slot) and suppressing others. This effectively modulates the overall attention sharpness, enabling the model to switch between diffuse (low confidence, high entropy) and focused (high confidence, low entropy) modes—a dynamic property central to softmax attention but absent in standard linear variants.

\subsection{Asymptotic Winner-Take-All Dynamics}
\label{sec:analysis:wta}

Building on the restored magnitude sensitivity, we show that SLA's dual gating mechanism asymptotically converges to a strict ``Winner-Take-All'' selection. This ensures that information is routed only when the query and key agree on the same semantic subspace.

\begin{theorem}
    \label{thm:wta}
    Let \(s^Q = Q W_{GQ}\) and \(s^K = K W_{GK}\) be the head projection scores for a query \(Q\) and key \(K\). Consider the scaled gates \(\mathcal{G}^Q(\lambda s^Q) = \softmax(\lambda s^Q)\) and \(\mathcal{G}^K(\lambda s^K) = \softmax(\lambda s^K)\).
    Define the information flow coefficient as \(C(\lambda) = \sum_{h=1}^H \mathcal{G}^Q_h(\lambda s^Q) \mathcal{G}^K_h(\lambda s^K)\).
    As \(\lambda \to \infty\):
    \begin{equation}
      \lim_{\lambda \to \infty} C(\lambda) = \delta(\operatorname*{argmax}_h s^Q_h, \operatorname*{argmax}_h s^K_h)
    \end{equation}
    where \(\delta(\cdot, \cdot)\) is the Kronecker delta.
\end{theorem}

This result mathematically confirms that SLA acts as a precise switch: it requires a ``consensus'' between the reading head (Query) and the writing head (Key) to enable information flow, effectively filtering out noise from mismatched subspaces. This mechanism serves as a powerful proxy for the selective attention of full softmax, achieving similar competitive dynamics via head-level routing.

\section{Experiments}
\label{sec:experiments}

\subsection{Instantiations on Linear Baselines}
\label{sec:method:instantiations}

We assess the versatility of SLA by instantiating it on top of three state-of-the-art linear attention architectures: RetNet~\citep{sun2023retentive}, Gated Linear Attention (GLA)~\citep{yang2023gated}, and Gated DeltaNet (GDN)~\citep{yang2024deltanet}.

\paragraph{Softmax-RetNet.}
RetNet employs a decay factor \(\gamma\) in its recurrent update to enforce locality. In Softmax-RetNet, we retain this decay mechanism within each head while modulating the head outputs via our proposed softmax gate:
\begin{align}
    S_{h, t} &= \gamma_h S_{h, t-1} + (\mathcal{G}^K_{h,t} \cdot \phi(k_{h,t}))^\top v_{h,t}, \\
    y_{h, t} &= (\mathcal{G}^Q_{h,t} \cdot \phi(q_{h,t})) S_{h, t}.
\end{align}
Here, \(\gamma_h\) encodes positional distance, whereas our gates \(\mathcal{G}^Q\) and \(\mathcal{G}^K\) capture semantic relevance through competition.

\paragraph{Softmax-GLA.}
Gated Linear Attention (GLA) introduces a data-dependent forget gate \(\alpha_t\) to the recurrence. The instantiation of Softmax-GLA is straightforward:
\begin{align}
    S_{h, t} &= \alpha_{h,t} \odot S_{h, t-1} + (\mathcal{G}^K_{h,t} \cdot \phi(k_{h,t}))^\top v_{h,t}, \\
    y_{h, t} &= (\mathcal{G}^Q_{h,t} \cdot \phi(q_{h,t})) S_{h, t}.
\end{align}
In this formulation, the original gate \(\alpha_{h,t}\) governs memory \emph{maintenance} (decay), while our softmax gates \(\mathcal{G}^Q\) and \(\mathcal{G}^K\) govern \emph{information selection} (input/output routing).

\paragraph{Softmax-GatedDeltaNet.}
Gated DeltaNet~\citep{yang2025gateddeltanetworksimproving} uses a delta rule for memory updates, conceptually replacing addition with rewriting. In Softmax-GatedDeltaNet, we apply the head-softmax over the output of the delta-update mechanism:
\begin{align}
    v'_{\text{new}} &= \beta_t \odot (\mathcal{G}^K_{h,t} \cdot v_t - \phi(k_t) S_{t-1}), \\
    S_{h, t} &= S_{h, t-1} + (\phi(k_{h,t}))^\top v'_{h, \text{new}}, \\
    y_{h, t} &= (\mathcal{G}^Q_{h,t} \cdot \phi(q_{h,t})) S_{h, t}.
\end{align}
By introducing competition to Gated DeltaNet, we allow the model to selectively attend to the most effective ``rewritten'' memory states, combining precise memory control with global selection.

\subsection{Experimental Setup}
\label{subsec:setup}

\paragraph{Baselines.}
In addition to the aforementioned linear baselines (RetNet, GLA, and Gated DeltaNet), we include Transformer++~\citep{Transformer++} as a strong full-attention baseline to benchmark the performance gap between linear and quadratic attention mechanisms.

\paragraph{Model Configuration.}
For a fair comparison, all models are trained from scratch at the 340M-parameter scale. Following prior work~\citep{yang2023gated}, we use a unified backbone configuration with 24 layers, a hidden size of 1024, and 4 attention heads. Unless otherwise specified, all other architectural and optimization settings are kept identical across models.

We implement the token-mixing blocks consistent with their original formulations. Specifically, for GLA and RetNet, we follow~\citep{yang2023gated} and apply no additional non-linear activation functions \(\phi(\cdot)\) to \(Q\) and \(K\) (i.e., identity mapping). For GDN, we adopt the short-convolution module and use SiLU~\citep{elfwing2017sigmoidweightedlinearunitsneural} as the activation function, in accordance with~\citep{yang2025gateddeltanetworksimproving}.

\paragraph{Training Details.}
We sample 15B tokens from SlimPajama~\citep{cerebras2023slimpajama} and tokenize it using the Mistral tokenizer for training, following the recipe of~\citep{yang2023gated}. We use AdamW with a peak learning rate of \(1\times10^{-3}\), weight decay of 0.1, and gradient clipping with a maximum norm of 1.0. The learning rate follows a cosine decay schedule with a 0.5B-token warmup. The maximum sequence length is set to 4096. All models are trained on 8 H20 GPUs. Implementations are based on the open-source Triton-based FLA library~\citep{yang2024fla}.

\subsection{Sparse Retrieval Capabilities}
\label{subsec:sparse_retrieval}

A central motivation of SLA is to restore the ``winner-take-all'' competition characteristic of softmax attention, which is notably absent in standard linear variants. This lack of competition impedes the model's ability to sharply focus on relevant information, particularly in recall-intensive scenarios requiring robust noise suppression. To rigorously evaluate whether SLA successfully reinstates this capability, we assess its performance across the following two dimensions.

\iffalse
three distinct retrieval-oriented dimensions:
\begin{enumerate}
    \item \textbf{Real-World Retrieval:} Evaluates the ability to extract specific knowledge from documents in practical QA settings.    
    \item \textbf{Needle-In-A-Haystack (NIAH):} Probes the limits of precise retrieval by requiring the identification of a single piece of information embedded within a large corpus of irrelevant text.
    \item \textbf{Long-Context Understanding:} Tests the model's capacity to maintain retrieval accuracy and comprehension over extended sequences (up to 4K tokens) where distractors are abundant.
\end{enumerate}
\fi

\begin{table}[t]
\vspace{0mm}
\centering
%\scriptsize
\setlength{\tabcolsep}{2pt}
\caption{Accuracy on real-world retrieval tasks.}
\begin{tabular}{l|cccc|c}
\hline
\textbf{Model} & \textbf{SWDE} & \textbf{SQuAD} & \textbf{FDA} & \textbf{Avg.} $\uparrow$ & \textbf{$\Delta$Avg.} \\
\hline
Transformer++     & 52.21 & 30.90 & 65.43 & 49.51 & -- \\
\hline
RetNet& 19.71& 27.28& \textbf{12.89}& 19.96& -- \\
\rowcolor{slablue} \textit{Softmax RetNet}& \textbf{30.51}& \textbf{32.74}& 9.98& \textbf{24.41}& +4.45\\ \hline
GLA               & 22.41 & 25.84 &  9.26 & 19.17 & -- \\
\rowcolor{slablue} \textit{Softmax GLA}& \textbf{33.48} & \textbf{31.27} & \textbf{15.88} & \textbf{26.88} & +7.71 \\ \hline
GDN    & 41.40 & 34.05 & 29.13 & 34.86 & -- \\
\rowcolor{slablue} \textit{Softmax GDN}& \textbf{41.80} & \textbf{34.76} & \textbf{28.96} & \textbf{35.17} & +0.31 \\
\hline
\end{tabular}
\addtolength{\tabcolsep}{2.5pt}
\label{tab:recall_results_small}
\vspace{-4mm}
\end{table}

\begin{table*}[h]
\centering
\caption{Zero-shot performance comparison on S-NIAH benchmark}
\resizebox{1\textwidth}{!}{%
\begin{tabular}{ll|cccc|cccc|cccc}
\toprule
& & \multicolumn{4}{c|}{S-NIAH-1} & \multicolumn{4}{c|}{S-NIAH-2} & \multicolumn{4}{c}{S-NIAH-3} \\
& & \multicolumn{4}{c|}{(pass-key retrieval)} & \multicolumn{4}{c|}{(number in haystack)} & \multicolumn{4}{c}{(uuid in haystack)} \\
\cmidrule{3-14}
Model & & 1K & 2K & 4K & 8K & 1K & 2K & 4K & 8K & 1K & 2K & 4K & 8K \\
\midrule
Transformer++   & & 100.0 & 100.0 & 100.0 & 0.0  & 100.0 & 100.0 & 100.0 & 0.0  & 95.2 & 91.0 & 68.0 & 0.0 \\
\hline
RetNet      & & 71.6 & 27.8 & 14.8 & 6.6& 88.6 & 32.0  & 12.0  & 7.2 & 3.6 & 2.2 & 0.8 & \textbf{1.8} \\
\rowcolor{slablue} \textit{Softmax RetNet} & & 99.2 & 71.2 & 18.0  & 0.0& 95.4 & 69.2  & 13.8  & 3.0 & \textbf{21.2} & \textbf{6.4} & \textbf{1.6} & 0.0 \\ \hline
GLA           & & 95.2  & 44.8  & 19.8  & 8.4  & 97.0  & 81.4  & 27.2  & 3.0  & \textbf{30.4} & \textbf{19.2} & 0.2  & 0.6 \\
\rowcolor{slablue} \textit{Softmax GLA}   & & 98.6  & 85.2  & 40.4  & 13.6 & 99.4  & 98.4  & 62.6  & 17.4 & 28.6 & 15.4 & \textbf{5.6 } &\textbf{ 1.6} \\ \hline
GDN           & & 100.0 & 100.0 & 100.0 & 100.0& 100.0 & 100.0 & 67.8  & 14.4 & 3.6  & 16.4 & 0.0  & 5.9 \\
\rowcolor{slablue} \textit{Softmax GDN}  & & 100.0 & 100.0 & 100.0 & 100.0 & 100.0 & 100.0 & 57.0  & 37.4 & \textbf{89.4} & \textbf{89.6} & \textbf{46.6} & \textbf{21.4} \\
\bottomrule
\end{tabular}
}
\label{tab:niah-results}
\end{table*}

\paragraph{Real-World Retrieval.}
We first examine whether head-wise softmax competition translates into stronger evidence retrieval in realistic scenarios. Following the protocol of~\citep{arora2025simplelinearattentionlanguage}, we conduct zero-shot in-context learning on three benchmarks: FDA~\citep{wu2021medical}, SWDE~\citep{lockard-etal-2019-openceres}, and SQuAD~\citep{rajpurkar2018know}. These tasks require the model to identify and utilize specific spans from the context to answer queries. The inputs are truncated to 4K tokens.

As presented in Table~\ref{tab:recall_results_small}, the full-attention Transformer++ achieves the strongest overall performance, serving as an upper bound. Among linear baselines, incorporating SLA yields significant gains. Notably, Softmax-GLA demonstrates a substantial improvement, boosting average accuracy from 19.17\% to 26.88\% (+7.71\%). RetNet also sees a marked increase (+4.45\%), while GDN shows consistent performance with a slight uptick. These results confirm that reintroducing competition via SLA effectively sharpens the model's focus, enabling more accurate retrieval of relevant evidence in standard QA tasks.

\paragraph{Needle-In-A-Haystack (NIAH).}
To further stress-test precise retrieval capabilities, we employ the Needle-In-A-Haystack benchmark from RULER~\citep{hsieh2024rulerwhatsrealcontext}. This task requires the model to recover a specific value (the ``needle'') associated with a query key, buried within a long sequence of distractor text (the ``haystack''). It challenges both the model's long-range memory retention and its ability to robustly filter out interference.

Results in Table~\ref{tab:niah-results} reveal that while the quadratic Transformer++ fails at 8K length, linear baselines degrade significantly as context grows. In contrast, adding SLA consistently enhances robustness. This improvement is particularly pronounced in the most challenging setting, S-NIAH-3 (UUID retrieval), where standard linear models struggle. Notably, Softmax-GDN achieves a remarkable performance boost, significantly outperforming its baseline. This demonstrates that the cross-head softmax competition empowers the model to more reliably distinguish the target signal from background noise, effectively mitigating the ``loss of focus'' issue common in linear attention.

\subsection{Basic Language Modeling Capabilities}

\begin{table*}[t!]
\vspace{0mm}
\centering
%\scriptsize
%\addtolength{\tabcolsep}{-2.5pt}
\caption{Performance comparison on language modeling and zero-shot common-sense reasoning evaluated by lm-evaluation-harness~\citep{eval-harness} }
\begin{tabular}{l|cc||cccccc|cc}
\toprule
\textbf{Model}  & \textbf{Wiki.}  &  \textbf{LMB.} &  \textbf{LMB.} & \textbf{PIQA} & \textbf{Hella.} & \textbf{Wino.} & \textbf{ARC-e} & \textbf{ARC-c} & \textbf{Avg.} & \textbf{$\Delta$Avg.} \\
 & ppl $\downarrow$  &  ppl $\downarrow$  &  acc $\uparrow$  & acc $\uparrow$ & acc $\uparrow$ & acc $\uparrow$ & acc $\uparrow$ & acc $\uparrow$ & $\uparrow$ &  $\uparrow$ \\
\hline
Transformer++    & 24.59 & 31.26 & 34.39 & 65.07 & 31.89 & 52.57 & 46.55 & 19.62 & 41.68 & -- \\
\cline{1-11}
RetNet      & 29.58 & 42.88 & 30.93 & 63.93 & 31.11 & \textbf{51.62} & 45.62 & \textbf{19.71} & 40.49 & -- \\
\rowcolor{slablue} \textit{Softmax RetNet}  & \textbf{27.79 }& \textbf{40.76} & \textbf{31.83} & \textbf{65.07 }& \textbf{31.50} &50.99 & \textbf{46.21} & 19.28 & \textbf{40.81} & +0.32 \\ \hline
GLA            & 28.93 & 43.63 & 30.37 & 63.60 & 30.41 & \textbf{52.72 }& 40.87 & 18.43 & 39.40 & -- \\
\rowcolor{slablue} \textit{Softmax GLA}    & \textbf{26.32} & \textbf{39.67} & \textbf{32.58} & \textbf{64.15} & \textbf{31.82} & 50.20 & \textbf{45.66} &\textbf{ 20.14 }& \textbf{40.76} & +1.36 \\ \hline
GDN        & 24.22 & 32.94 & 33.34 & \textbf{65.29} & \textbf{32.37} & \textbf{51.14} & 47.77 & 18.52 & 41.41 & -- \\
\rowcolor{slablue} \textit{Softmax GDN}       & \textbf{23.99 }& \textbf{31.44} & \textbf{34.58} & 64.64 & 31.40 & 50.99 & \textbf{47.83 }& \textbf{20.82} & \textbf{41.71} & +0.30 \\
\bottomrule
\end{tabular}
%\addtolength{\tabcolsep}{2.5pt}
\label{tab:commonsense_results}
\end{table*}

\paragraph{Language Modeling.}
We first assess whether SLA enhances the fundamental next-token prediction capability of efficient attention backbones by restoring the missing competition mechanism. Following prior work~\citep{yang2023gated}, we report perplexity on WikiText~\citep{merity2016pointer} and LAMBADA~\citep{paperno2016lambadadatasetwordprediction}.

As shown in Table~\ref{tab:commonsense_results}, incorporating SLA consistently reduces perplexity across all evaluated backbones, indicating that the gain is architecture-agnostic rather than specific to a particular recurrence or gating design. In particular, GLA benefits the most, reducing WikiText perplexity from 28.93 to 26.32 and LAMBADA perplexity from 43.63 to 39.67. RetNet and GDN exhibit similar trends with consistent improvements. This improvement aligns with our motivation: feature-wise decomposition in linear attention removes token-wise softmax normalization, leading to a lack of global competition and potential context collapse. By lifting softmax to the head dimension, SLA forces attention heads (acting as coarse semantic slots) to compete via \(\mathcal{G}^Q\) and \(\mathcal{G}^K\), enabling a ``winner-take-all'' style selection of relevant subspaces while suppressing noise. Consequently, the model allocates its limited recurrent state capacity more selectively, improving likelihood and lowering perplexity.

\paragraph{Zero-shot Commonsense Reasoning.}
We further test whether the competition restored by SLA transfers beyond likelihood to short-context zero-shot commonsense reasoning, where the model must discriminate between plausible options using limited evidence. We evaluate accuracy on a suite of standard benchmarks, including LAMBADA~\citep{paperno2016lambadadatasetwordprediction}, PIQA~\citep{Bisk2020}, HellaSwag~\citep{zellers2019hellaswag}, WinoGrande~\citep{sakaguchi2019winogrande}, ARC-Easy, and ARC-Challenge~\citep{Clark2018ThinkYH}.

Table~\ref{tab:commonsense_results} shows that SLA improves the average accuracy for multiple backbones, with the strongest gains observed on GLA. RetNet and GDN also see consistent, albeit smaller, improvements. These results validate SLA's design: head-wise softmax competition encourages sharper, more discriminative evidence aggregation, which is crucial for multiple-choice tasks where distinguishing the correct option from near-misses is key. %While gains vary across tasks due to differing reasoning requirements, the consistent improvement in average accuracy suggests that reinstating global competition at an inter-head granularity enhances both language modeling and downstream reasoning.

\begin{figure}[t]
  \centering
  \begin{minipage}[t]{0.48\linewidth}
    \centering
    \includegraphics[width=\linewidth]{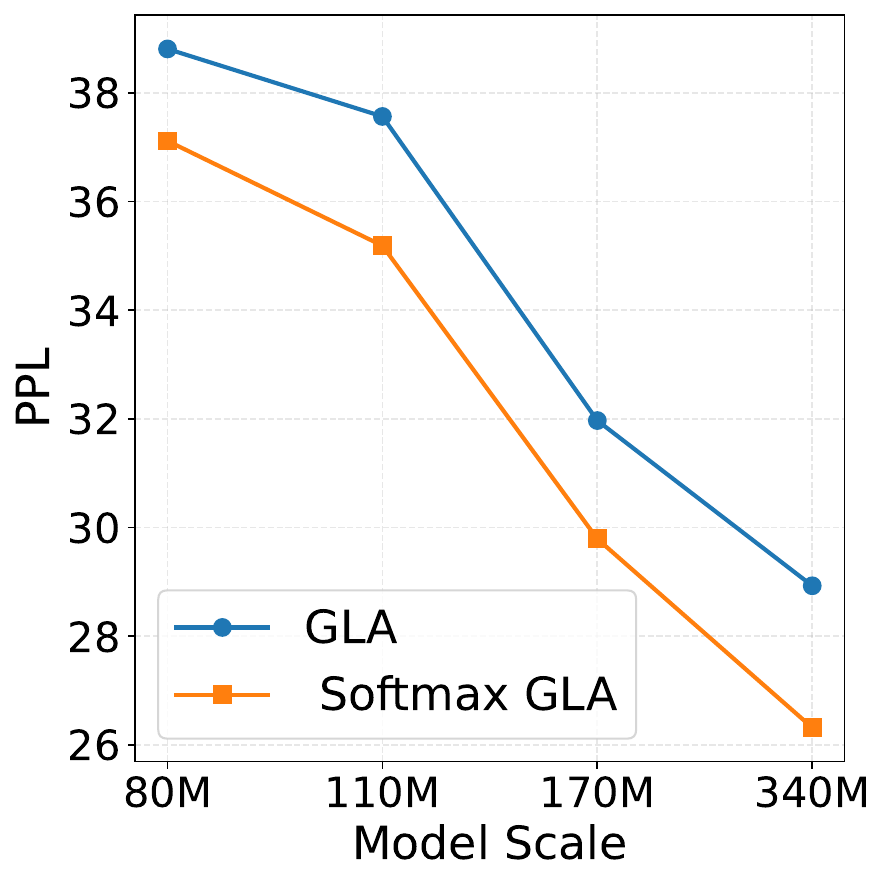}
    \caption{Scaling curves on WikiText perplexity for GLA and Softmax GLA.}
    \label{fig:scale}
  \end{minipage}\hfill
  \begin{minipage}[t]{0.48\linewidth}
    \centering
    \includegraphics[width=\linewidth]{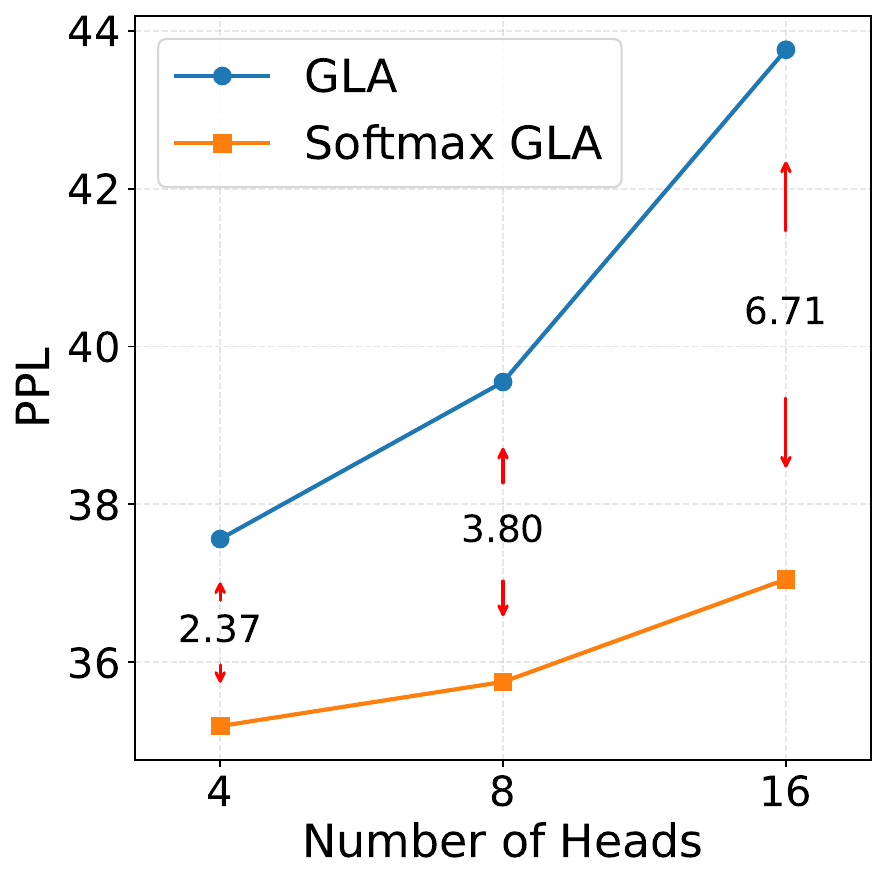}
    \caption{WikiText perplexity under different numbers of attention heads.}
    \label{fig:head}
  \end{minipage}
\end{figure}

\begin{figure*}[t]
  \centering
  \begin{subfigure}[t]{0.28\textwidth}
    \centering
    \includegraphics[width=\linewidth]{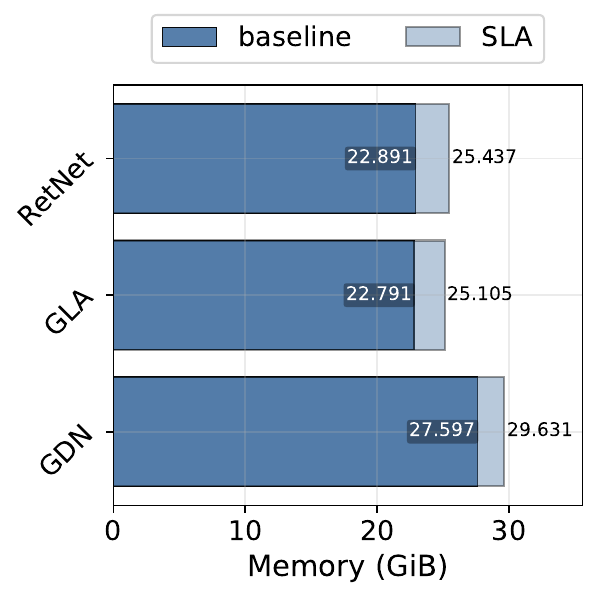}
    \caption{Memory footprint}
    \label{fig:mem}
  \end{subfigure}\hfill
  \begin{subfigure}[t]{0.38\textwidth}
    \centering
    \includegraphics[width=\linewidth]{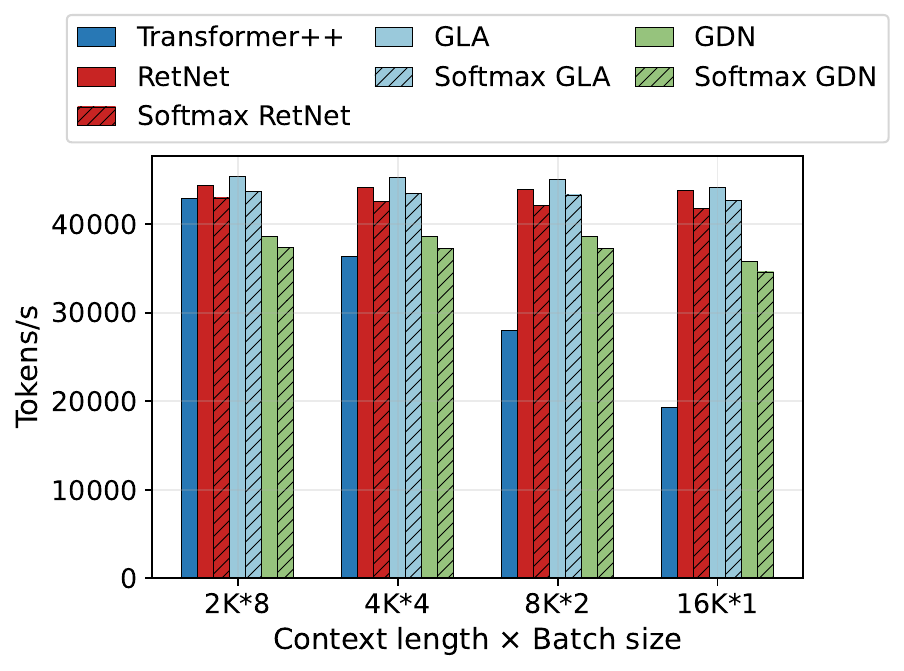}
    \caption{Training throughput}
    \label{fig:train}
  \end{subfigure}\hfill
  \begin{subfigure}[t]{0.33\textwidth}
    \centering
    \includegraphics[width=\linewidth]{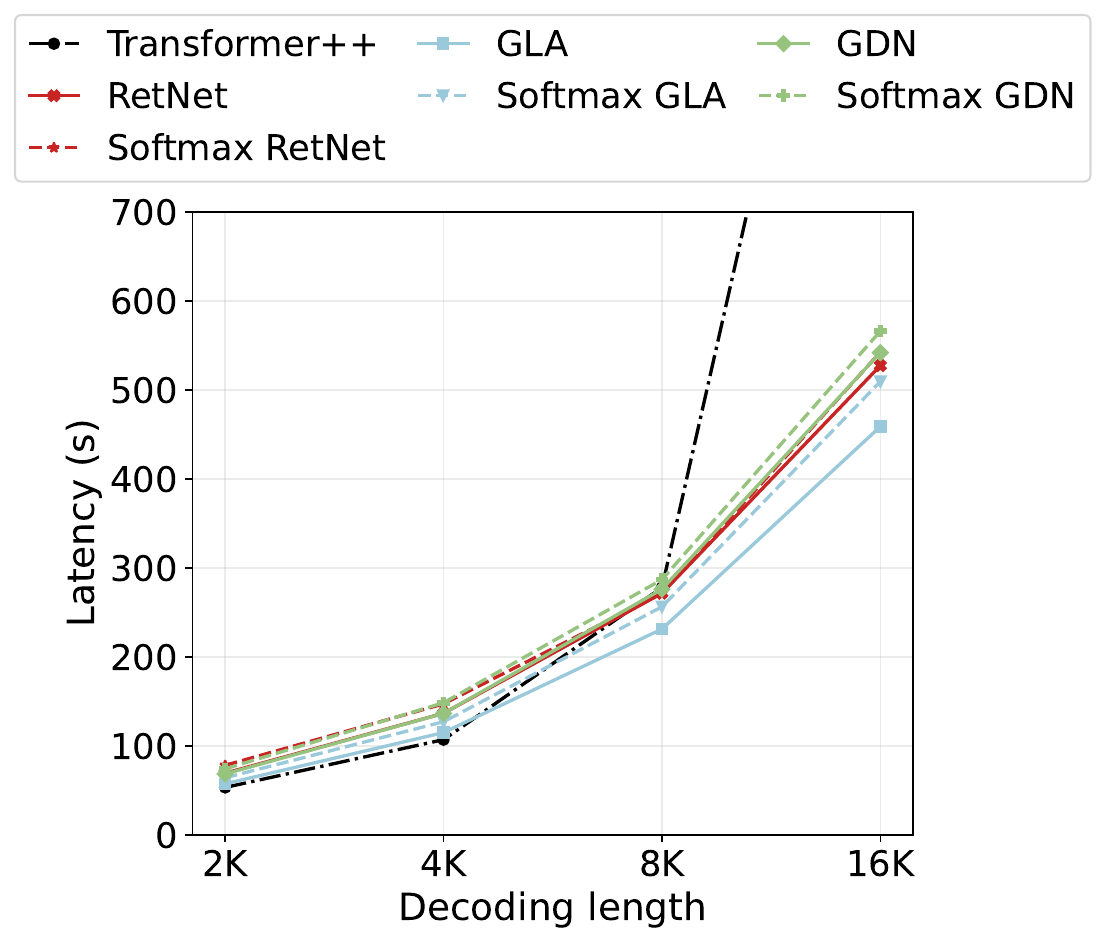}
    \caption{Inference latency}
    \label{fig:lat}
  \end{subfigure}

  \caption{Overall comparison on training throughput, memory footprint and inference latency.}
  \label{fig:overall}
\end{figure*}

\subsection{Scaling Analysis}
\label{subsec:scale}

To investigate whether the benefits of Softmax Linear Attention persist across different model scales, we conducted a scaling analysis on the GLA backbone. We trained models with parameters ranging from 80M to 340M on 15B tokens.

As illustrated in Figure~\ref{fig:scale}, both standard GLA and Softmax-GLA exhibit predictable improvements in perplexity as model size increases, adhering to standard scaling laws. Crucially, Softmax-GLA consistently outperforms the standard GLA baseline across all evaluated scales. The performance gap remains significant even as the model size grows, indicating that the advantages of the proposed inter-head competition mechanism are not limited to small-scale models but are intrinsic to the architecture. A striking observation is the parameter efficiency gain; Softmax GLA at 170M achieves a perplexity (\(\approx\) 29.8) that approaches the performance of the standard GLA at 340M (\(\approx\) 29.0). %We expect that further per-scale hyperparameter tuning could lead to an even stronger scaling trend for Softmax-GLA.

\subsection{Impact of Head Granularity}
\label{subsec:head}

To validate our core hypothesis that SLA leverages the multi-head architecture as competitive semantic slots, we conducted an ablation study on the number of attention heads $H$. We trained models with $H \in \{4, 8, 16\}$ while keeping the total model parameters fixed (110M, implying a corresponding decrease in head dimension).

Figure~\ref{fig:head} reveals a striking trend: the performance gap between Softmax GLA and the baseline widens significantly as the number of heads increases. Specifically, the perplexity improvement provided by SLA grows from 2.37 at $H=4$ to 6.71 at $H=16$. This trend aligns perfectly with the design intuition behind SLA: we leverage attention heads as distinct semantic slots to store coarse-grained information and enforce competition among them. With more heads, the model obtains a richer set of semantic subspaces to represent diverse features. SLA effectively capitalizes on this by using the inter-head softmax to accurately route information to the most relevant slots and suppress noise. In contrast, standard GLA lacks this global selection mechanism; as the feature space becomes more fragmented with higher head counts, it struggles to aggregate these dispersed signals, leading to performance degradation.

\subsection{Training Efficiency}
\label{subsec:efficiency}

Figure~\ref{fig:overall} reports the efficiency of different backbones with and without SLA on a single H20 GPU, focusing on three aspects: peak training memory, training throughput, and autoregressive decoding latency.

\textbf{Peak memory footprint.} As shown in Figure~\ref{fig:mem}, incorporating SLA incurs a marginal additional peak memory cost during training for all backbones. In practice, this memory overhead remains manageable because the routing signals are low-dimensional and can be computed in reduced precision, and because the backbone activations dominate memory usage in long-context settings.

\textbf{Training throughput.} As shown in Figure~\ref{fig:train}, adding SLA leads to a slight but consistent throughput drop across various context-length and batch-size configurations. This overhead stems primarily from computing routing scores, applying the corresponding mixing/gating operations, and the additional element-wise softmax normalization. Importantly, the throughput degradation does not scale drastically with sequence length, indicating that the additional computation is lightweight compared to the dominant attention cost of the backbone.

\textbf{Inference latency.} Figure~\ref{fig:lat} compares end-to-end decoding latency under increasing decoding lengths. SLA introduces a mild latency increase that tracks the training-time overhead: the routing computation is performed per step and adds a small constant factor. Notably, the relative gap between the baseline and SLA remains stable across decoding lengths for GDN, suggesting that SLA does not introduce unfavorable length-dependent complexity. Overall, SLA improves modeling capacity while incurring only moderate efficiency overhead, enhancing its practicality for both training and long-context inference.

\section{Conclusion}
\label{sec:conclusion}

In this paper, we identify the lack of global competition as a key limitation of linear attention and propose Softmax Linear Attention to address it. By introducing a lightweight dual gating mechanism at the head level, SLA restores the ``winner-take-all'' dynamics essential for precise retrieval without sacrificing linear complexity. Theoretical analysis and experiments demonstrate that SLA consistently enhances SOTA linear baselines in language modeling and long-context tasks, offering a seamless and efficient solution to bridge the expressivity gap between linear and quadratic attention.

\section*{Impact Statement}
This paper presents work whose goal is to advance the field of machine learning. There are many potential societal consequences of our work, none of which we feel must be specifically highlighted here.

\bibliography{softmax_linear_attention}
\bibliographystyle{icml2026}

%%%%%%%%%%%%%%%%%%%%%%%%%%%%%%%%%%%%%%%%%%%%%%%%%%%%%%%%%%%%%%%%%%%%%%%%%%%%%%%
%%%%%%%%%%%%%%%%%%%%%%%%%%%%%%%%%%%%%%%%%%%%%%%%%%%%%%%%%%%%%%%%%%%%%%%%%%%%%%%
% APPENDIX
%%%%%%%%%%%%%%%%%%%%%%%%%%%%%%%%%%%%%%%%%%%%%%%%%%%%%%%%%%%%%%%%%%%%%%%%%%%%%%%
%%%%%%%%%%%%%%%%%%%%%%%%%%%%%%%%%%%%%%%%%%%%%%%%%%%%%%%%%%%%%%%%%%%%%%%%%%%%%%%
\newpage
\appendix
\onecolumn
\section{Derivation of Recurrent Implementation}
\label{sec:appendix:recurrent}

The recurrent form of Softmax Linear Attention (SLA) follows directly from the standard linear attention recurrence by absorbing the scalar gates into the feature maps. This derivation establishes the mathematical equivalence between the parallel formulation (Eq.~\ref{eq:sla_final}) and the recurrent update rules (Eq.~\ref{eq:sla_update}).

Recall the parallel form for head \(h\) at time \(t\) (derived from Eq.~\ref{eq:sla_final}):
\begin{equation}
    y_{h,t} = \mathcal{G}^Q_{h,t} \phi(q_{h,t}) \sum_{j=1}^{t} (\mathcal{G}^K_{h,j} \phi(k_{h,j}))^\top v_{h,j}.
\end{equation}
By defining the \textit{gated} feature maps as \(\tilde{q}_{h,t} = \mathcal{G}^Q_{h,t} \phi(q_{h,t})\) and \(\tilde{k}_{h,t} = \mathcal{G}^K_{h,t} \phi(k_{h,t})\), the equation simplifies to the standard linear attention form:
\begin{equation}
    y_{h,t} = \tilde{q}_{h,t} \sum_{j=1}^{t} \tilde{k}_{h,j}^\top v_{h,j}.
\end{equation}
Defining the recurrent state \(S_{h,t} = \sum_{j=1}^{t} \tilde{k}_{h,j}^\top v_{h,j}\), we immediately obtain the constant-memory update rule:
\begin{equation}
    S_{h,t} = S_{h,t-1} + \tilde{k}_{h,t}^\top v_{h,t}, \quad y_{h,t} = \tilde{q}_{h,t} S_{h,t}.
\end{equation}
This exactly matches the recurrent updates in Eq.~\ref{eq:sla_update}, confirming that the head-level gating does not disrupt the linear recurrence and maintains an inference cost of \(O(1)\) with respect to sequence length.

\section{Proof of Theorem \ref{thm:wta}}
\label{proof:wta}

\begin{proof}
    As \(\lambda \to \infty\), the softmax function converges to a one-hot vector (assuming unique maximums). Let \(h_Q = \operatorname*{argmax}_h s^Q_h\) and \(h_K = \operatorname*{argmax}_h s^K_h\). Then \(\mathcal{G}^Q_h \to \mathbb{I}[h = h_Q]\) and \(\mathcal{G}^K_h \to \mathbb{I}[h = h_K]\).
    The term \(\mathcal{G}^Q_h \mathcal{G}^K_h\) is non-zero (approaching 1) if and only if \(h = h_Q\) and \(h = h_K\), i.e., \(h_Q = h_K\).
    Thus, the sum \(\sum_h \mathcal{G}^Q_h \mathcal{G}^K_h\) converges to 1 if \(h_Q = h_K\), and 0 otherwise.
\end{proof}

\begin{table}[t]
  \centering
  \begin{minipage}[t]{0.55\linewidth}
    \vspace{0pt}
    \centering
    \caption{Model hyperparameters for Scaling Analysis}
    \label{tab:scale_parameters}
    \begin{tabular}{ccccc}
      \toprule
      Params Scale & Heads & Layers & Hidden\_size & Tokens \\
      \midrule
      80M  & 4 & 16 & 640  & 15B \\
      110M & 4 & 18 & 640  & 15B \\
      170M & 4 & 24 & 768  & 15B \\
      340M & 4 & 24 & 1024 & 15B \\
      \bottomrule
    \end{tabular}
  \end{minipage}\hfill
  \begin{minipage}[t]{0.45\linewidth}
    \vspace{0pt}
    \centering
    \includegraphics[width=\linewidth]{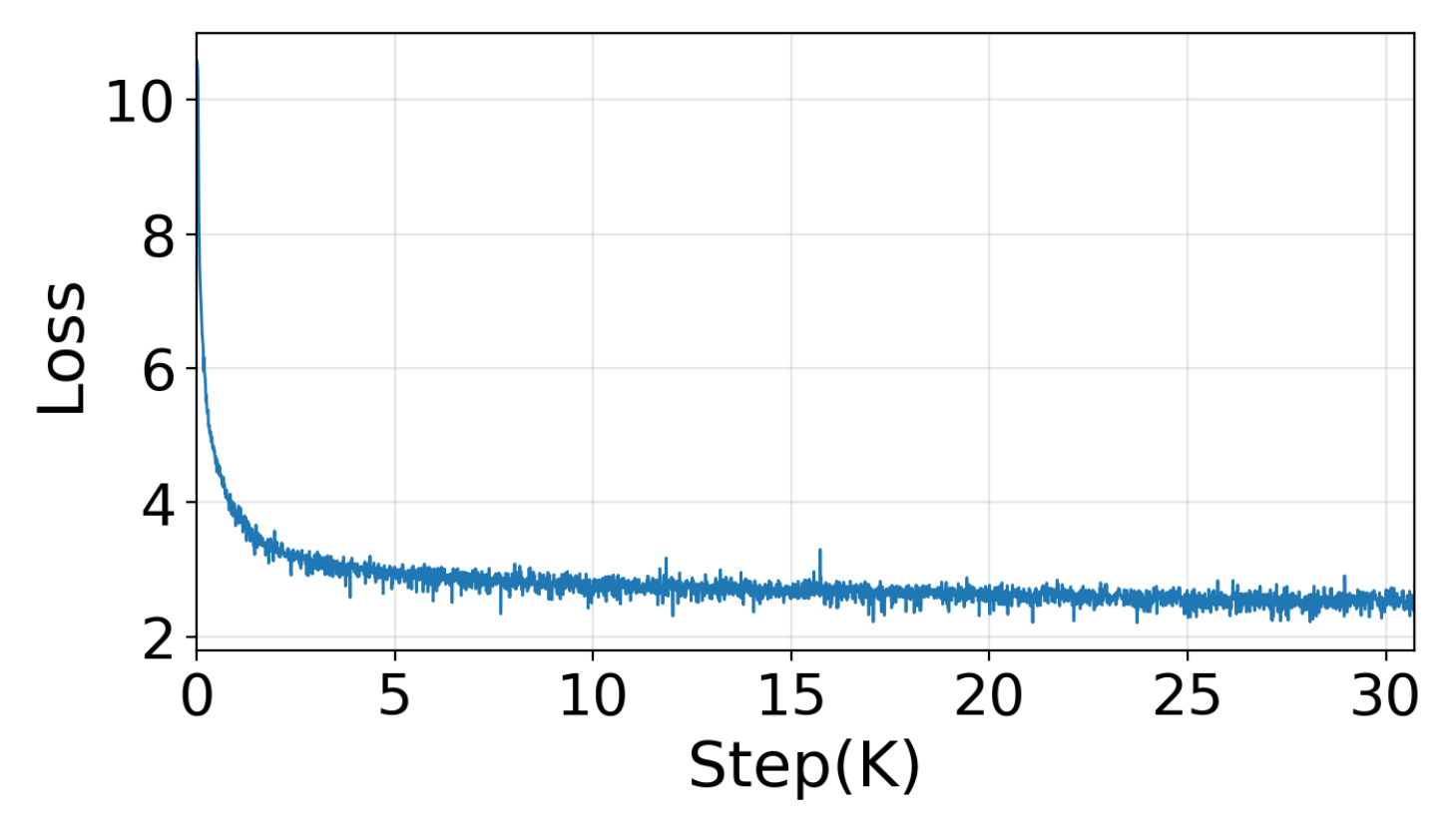}
    \captionof{figure}{Softmax RetNet training loss curves under the same setup in Section~\ref{subsec:setup}.}
    \label{fig:train_loss}
  \end{minipage}
\end{table}

\section{Model hyperparameters for Scaling Analysis}
Table~\ref{tab:scale_parameters} summarizes the hyperparameters used in Section~\ref{subsec:scale}. All models are trained on the same 15B tokens with the identical training setup in Section~\ref{subsec:setup}. We scale model size by increasing layers and hidden size while keeping the number of attention heads fixed ($H=4$) for a controlled comparison.

\section{Training Stability}
As shown in Figure~\ref{fig:train_loss}, under the same training setup in Section~\ref{subsec:setup}, introducing SLA keeps the training loss trajectory smooth, without introducing noticeable extra oscillations, indicating that SLA preserves training stability.

%%%%%%%%%%%%%%%%%%%%%%%%%%%%%%%%%%%%%%%%%%%%%%%%%%%%%%%%%%%%%%%%%%%%%%%%%%%%%%%
%%%%%%%%%%%%%%%%%%%%%%%%%%%%%%%%%%%%%%%%%%%%%%%%%%%%%%%%%%%%%%%%%%%%%%%%%%%%%%%

\end{document}